\documentclass[11pt,english]{article}
\usepackage[T1]{fontenc}
\usepackage[latin9]{luainputenc}
\usepackage{color}
\usepackage{babel}
\usepackage{float}
\usepackage{units}
\usepackage{mathtools}
\usepackage{amsmath}
\usepackage{amsthm}
\usepackage{amssymb}
\usepackage{geometry}
\usepackage[pdfusetitle,
 bookmarks=true,bookmarksnumbered=false,bookmarksopen=false,
 breaklinks=false,pdfborder={0 0 1},backref=false,colorlinks=true]
 {hyperref}

\makeatletter

\floatstyle{ruled}
\newfloat{algorithm}{tbp}{loa}
\providecommand{\algorithmname}{Algorithm}
\floatname{algorithm}{\protect\algorithmname}

\theoremstyle{plain}
\newtheorem{thm}{\protect\theoremname}
\theoremstyle{plain}
\newtheorem{lem}[thm]{\protect\lemmaname}
\theoremstyle{definition}
\newtheorem{defn}[thm]{\protect\definitionname}

\@ifundefined{date}{}{\date{}}
\makeatother

\providecommand{\definitionname}{Definition}
\providecommand{\lemmaname}{Lemma}
\providecommand{\theoremname}{Theorem}

\begin{document}
\global\long\def\R{\mathbb{R}}%
\global\long\def\E{\mathbb{E}}%
\global\long\def\F{{\cal F}}%
\global\long\def\one{\mathbb{I}}%
\global\long\def\P{\mathbb{P}}%
\global\long\def\VV{V_{\lambda}^{k}}%
\global\long\def\N{\mathbb{N}}%

\title{Sample-efficient Multiclass Calibration under $\ell_{p}$ Error}
\author{Konstantina Bairaktari\thanks{Khoury College of Computer Sciences, Northeastern University. Supported
by NSF awards CCF-2311649, CNS-2232692 and CNS-2247484.} $\qquad$ Huy L. Nguyen\thanks{Khoury College of Computer Sciences, Northeastern University. Supported
by NSF award CCF-2311649 and a gift from Apple Inc.}}
\maketitle
\begin{abstract}
Calibrating a multiclass predictor, that outputs a distribution over
labels, is particularly challenging due to the exponential number
of possible prediction values. In this work, we propose a new definition
of calibration error that interpolates between two established calibration
error notions, one with known exponential sample complexity and one
with polynomial sample complexity for calibrating a given predictor.
Our algorithm can calibrate any given predictor for the entire range
of interpolation, except for one endpoint, using only a polynomial
number of samples. At the other endpoint, we achieve nearly optimal
dependence on the error parameter, improving upon previous work. A
key technical contribution is a novel application of adaptive data
analysis with high adaptivity but only logarithmic overhead in the
sample complexity.
\end{abstract}

\section{Introduction}

Trustworthiness and interpretability have become key concerns for
machine learning models, especially as they are increasingly used
for critical decision making. Calibration is an important tool, dating
back to classical forecasting literature \cite{dawid1982well,foster1998asymptotic},
that can be used to address some of these concerns. A predictor $h$
for binary classification that outputs values in $[0,1]$ is calibrated
if, among the inputs $x$ for which $h(x)=q$, exactly $q$ fraction
of them have a positive outcome. In recent years, a large body of
work has focused on developing algorithms that either learn calibrated
predictors or calibrate previously trained models. This notion has
also been extended to multi-calibration \cite{Hebert-JohnsonK18},
where the calibration guarantee holds for multiple, possibly overlapping
populations. Another important extension is to the multiclass setting,
as opposed to binary classification, which is the focus of this work.

Calibration presents two main challenges. The first is defining a
notion of calibration error that quantifies how much a predictor deviates
from being perfectly calibrated. This error metric must be testable
\cite{RosselliniSBRW25}, meaning that we should be able to detect
that a predictor has small error using a small number of samples.
While sharing common intuition, many different definitions of calibration
error exist in the literature. Typically, the predicted probabilities
are divided into bins and the calibration guarantee applies to conditioning
on the bins rather than on the predicted values. Some proposed error
metrics are not testable. For example, the $L_{\infty}$error as defined
in \cite{Gruber022}, measures the maximum conditional deviation between
the prediction and the true probability of the class across bins.
This maximum could occur in a bin containing points that appear with
very small probability, making it practically undetectable due to
insufficient sampling. The second challenge is developing algorithms
that efficiently learn a calibrated predictor from scratch or recalibrate
existing predictors, considering both sample complexity and computational
efficiency with respect to the problem parameters.

While discretizing the prediction space results in a reasonable number
of bins for binary classification, in the multiclass setting the number
of bins grows exponentially with the number of classes, presenting
a unique challenge. In fact, for a natural definition of the distance
to calibration, testing whether a given model is perfectly calibrated
requires the number of samples to be exponential in the number of
classes \cite{GopalanHR24}. Alternatively, other works have considered
a weaker definition where the predictor is considered calibrated if
the calibration error per bin is small, as opposed to measuring the
total error across all bins \cite{HaghtalabJ023,DworkLLT23}. In this
case, surprisingly, a calibrated predictor can be found using a polynomial
number of samples. A natural question is whether the weakening in
the definition is necessary and, if so, how much weakening is necessary
to remove the exponential dependence on the number of classes.

Calibration is important in its own right, but it is also desirable
for a predictor to be accurate. Given that most machine learning models
are developed using complex pipelines that are difficult to modify,
the ability to calibrate an existing model, as opposed to building
a new one from scratch, is valuable. This approach would allow one
to leverage the remarkable accuracy of existing models while adding
calibration guarantees. Moreover, it is possible for a predictor to
be calibrated yet uninformative. This underscores the importance of
maintaining accuracy alongside calibration. While many works in the
literature satisfy this requirement, the works with strong sample
complexity bounds above unfortunately do not \cite{HaghtalabJ023,DworkLLT23}.
Thus, a significant challenge is to develop efficient algorithms that
can calibrate a given predictor while making minimal targeted modifications.
Concretely, we aim to develop calibration algorithms for a given predictor
that satisfy the following two properties:
\begin{enumerate}
\item The resulting classifier is calibrated up to error $\varepsilon$.
\item The resulting classifier's accuracy remains within an additive error
of $\varepsilon$ compared to the accuracy of the given predictor
to allow for discretization and estimation error.
\end{enumerate}
In this work, we address the above questions and propose a new definition
of calibration error, which we call the $\ell_{p}$ calibration error.
This error notion is defined as the $\ell_{p}$ norm of the calibration
errors across all bins and classes. In particular, for a fixed bin
and class, we define the calibration error as the product of the absolute
difference between the expected value of the prediction and the true
probability of the class conditioned on the datapoint belonging to
the bin, and the probability mass of the bin. The definition that
adds up the errors across all bins and classes corresponds to the
special case $p=1$, also known as the expected calibration error
(ECE) \cite{GuoPSW17}, while the definition that measures the maximum
error across all bins and classes in \cite{HaghtalabJ023} corresponds
to $p=\infty$. As our measure of accuracy, we use the squared error
of the predictor. Our work shows that for all $p>1$, there exists
an algorithm that uses a polynomial number of samples in the number
of classes to calibrate any given predictor. For the special case
$p=\infty$ and a given desired calibration error $\varepsilon$,
the sample complexity is within a poly-logarithmic factor of $O\left(1/\varepsilon^{2}\right)$.
This is almost as good as one could hope for since even testing if
the fraction of data with positive outcome is $1/2$ or $1/2+\varepsilon$
already requires $\Omega\left(1/\varepsilon^{2}\right)$ samples.
\begin{thm}[{{[}Informal version of Theorem \ref{thm:main}{]}}]
 \label{thm:main-inf} There exists an algorithm that takes as input
any $k$-class predictor $f:\mathcal{X}\to\Delta_{k}$, runs in time
polynomial in $k$ and $\frac{1}{\varepsilon}$, and, using $\tilde{O}\left(\left(\frac{2^{1/(p-1)}}{\varepsilon^{p/(p-1)}}\right)^{2}\right)$
samples, returns a $k$-class predictor $h:\mathcal{X}\to\Delta_{k}$
that has:
\end{thm}

\begin{enumerate}
\item $\ell_{p}$ calibration error at most $\varepsilon$, and
\item squared error within an additive term $\tilde{O}\left(\frac{\varepsilon^{p/(p-1)}}{2^{1/(p-1)}}\right)$
from the squared error of $f$.
\end{enumerate}

\subsection{Our techniques}

When $p=\infty$, we observe that if a bin contains at most an $\varepsilon$
fraction of the data distribution, its calibration error for any class
is also bounded by $\varepsilon$. Thus, one only needs to care about
$1/\varepsilon$ bins with large probability masses. We generalize
this idea to all $\ell_{p}$ norms for $p>1$ and allow the algorithm
to focus only on bins with large probability masses. This observation
is sufficient to obtain a (large) polynomial sample complexity. This
approach works because our calibration error notion incorporates the
probability mass of the bin in the $p$-exponent, naturally assigning
higher weights to larger bins.

A second observation that further improves the sample complexity is
that for interpretability reasons the output of our calibrated predictor
should be probability distributions over the $k$ labels, a constraint
not enforced in previous work. This constraint significantly reduces
the discretized prediction space during calibration compared to $\lambda^{k}$
in prior works (where $\lambda$ is the number of discrete values
per coordinate), since the predictor outputs must form valid probability
distributions with coordinates summing to $1$. Consequently, our
set of bins approximately corresponds to the set of sparse vectors
in $k$ dimensions containing $\lambda$ non-zero elements, each equal
to $1/\lambda$. The crucial insight is that the number of such sparse
vectors is polynomial rather than exponential in $k$.

Calibrating the predictor might require adaptively merging many high-probability
bins together. Naively estimating the error of all subsets of high-probability
bins to $\varepsilon$ requires $1/\varepsilon^{3}$ samples (due
to the number of subsets being $\Omega\left(\exp\left(1/\varepsilon\right)\right)$).
Adaptive data analysis has been applied in previous works to reduce
the number of samples, but the overhead remains polynomial in $1/\varepsilon$.
Surprisingly, our algorithm is still highly adaptive, but with a novel
analysis, the overhead in the sample complexity is only logarithmic
in $1/\varepsilon$. Our techniques might be applicable to other problems
where adaptive data analysis is used.

\subsection{Related work}

The most closely related works are \cite{HaghtalabJ023,DworkLLT23}.
In the case where $p=\infty$, they showed that with access to an
oracle for the exact probabilities, $O\left(\varepsilon^{-2}\ln k\right)$
queries to the oracle suffice to find an $\varepsilon$-calibrated
predictor for $k$-class classification. These results construct a
new model from scratch and do not aim to preserve the accuracy of
a previously trained model, as our algorithm does. Furthermore, \cite{HaghtalabJ023}
showed that $O\left(\nicefrac{\ln(k)}{\varepsilon^{4}}\left(\ln(1/(\varepsilon))+\ln(V)\right)\right)$
samples suffice for their algorithm, where $V$ is the number of discretized
bins. In their case, $\ln(V)=O\left(k\ln(\lambda)\right)$, with $\lambda$
being a non-negative integer that controls the granularity of discretization.
In contrast, our algorithm employs a different discretization scheme
where $\ln(V)=O\left(\min\left(k,\lambda\right)\ln\left(\lambda+k\right)\right)$.
This alternative approach contributes to our improved sample complexity.
However, it introduces additional complexity to the algorithm due
to the need to project the predictions onto the probability simplex.
These projections impact both the calibration and the accuracy of
the predictor. For calibration, updating one coordinate of the predictor
and then projecting can alter other coordinates that are already calibrated.
For accuracy, we must carefully select the projection method to ensure
that the accuracy is preserved.

Iterative calibration algorithms inherently present an adaptive data
analysis challenge, due to the dependence of the bins that whose predictions
get updated on the current predictor. Most algorithms in this area,
including ours, perform $\text{poly}\left(1/\varepsilon\right)$ iterations.
Some works, such as \cite{GopalanKSZ22}, address the adaptivity issue
by resampling at each iteration to estimate the calibration error,
which results in a $\text{poly}\left(1/\varepsilon\right)$ overhead
in sample complexity. Other works, including \cite{HaghtalabJ023,Hebert-JohnsonK18},
use tools from adaptive data analysis to bound the sample complexity
in a black-box way. Specifically, they use the strong composition
property of differential privacy, which allows answering $t$ adaptive
queries with only a $\tilde{O}\left(\sqrt{t}\right)$overhead. As
a result, this method incurs a smaller $\text{poly}(1/\varepsilon)$
overhead in sample complexity. Our novel algorithm and analysis achieve
a tighter bound, requiring only a $\log(1/\varepsilon)$ overhead
in sample complexity. This significantly improves the overall sample
complexity of the iterative calibration process.

Due to the challenges of calibration in the multiclass setting, several
weaker error definitions have been proposed. A lot of work focuses
on calibrating existing neural networks. For instance, \cite{GuoPSW17}
introduced confidence calibration, where the conditioning is done
on the highest prediction value among all classes and explored several
methods including binning methods, matrix and vector scaling, and
temperature scaling. Related notions include top-label calibration
\cite{GuptaR22}, which conditions on the highest prediction value
and on the identity of the top class and class-wise calibration \cite{KullPKFSF19},
which conditions on individual class predictions rather than on the
entire probability vector. While extensive literature exists on $\ell_{p}$-style
calibration measures \cite{KumarLM19,VaicenaviciusWA19,WidmannLZ19,ZhangKH20,Gruber022,PopordanoskaSB22},
our approach differs fundamentally. Our $\ell_{p}$error definition
incorporates the probability mass of the bin in the $p$-exponent,
ensuring that bins with large error have also sufficient mass to be
detected, resolving the limitation that previously considered $\ell_{p}$
calibration error may require exponentially many samples for testing.
On the theoretical front, \cite{GopalanKSZ22} proposed low-degree
multi-calibration as a less-expensive alternative to the full requirement.
The work of \cite{GopalanHR24} introduced projected smooth calibration
as a new multiclass calibration definition that allows for efficient
algorithms while giving strong guarantees.

\section{Preliminaries}

\label{sec:Preliminaries}

We use $\mathcal{X}$ to denote the feature space and $[k]=\{1,\ldots,k\}$
to denote the label space. We also use the $k$-dimensional one-hot
encoding of a label as an equivalent representation. We use $\Delta_{k}$
to denote the probability simplex over $k$ labels. In this work,
a $k$-class predictor $f$ is a function that maps feature vectors
in $\mathcal{X}$ to distributions in $\Delta_{k}$. 

Instead of conditioning on the exact predicted probability vector,
we partition $\Delta_{k}$ into level sets. Previous methods partition
$\Delta_{k}$ by mapping the prediction vectors to the closest vector
in $L^{k}$, the $k$-ary Cartesian power of $L=\left\{ 0,1/\lambda,2/\lambda,\ldots,1\right\} $,
where $\lambda$ is a positive integer that determines the discretization
granularity. The coordinates of vectors in $L^{k}$ may not sum to
$1$. We use an alternative partition of $\Delta_{k}$ via a many-to-one
mapping onto $V_{\lambda}^{k}$. We define $V_{\lambda}^{k}$ to be
the subset of $L^{k}$ such that for every member $v$ of $V_{\lambda}^{k}$,
there exists a probability distribution $u\in\Delta_{k}$ such that
$v$ is obtained by rounding down every coordinate of $u$ to a multiple
of $1/\lambda$. Formally,

\[
V_{\lambda}^{k}=\left\{ v\in L^{k}:\exists u\in\Delta_{k}\text{s.t.}\left\lfloor u_{i}\lambda\right\rfloor /\lambda=v_{i}\ \forall i\in[k]\right\} .
\]
While vectors in $V_{\lambda}^{k}$ are not necessarily distributions,
they are close to vectors that are distributions. This property allows
$V_{\lambda}^{k}$ to be significantly smaller than $L^{k}$.
\begin{lem}
\label{lem:size-of-V}For any $\lambda,k\in\N^{+}$, the number of
level sets in $V_{\lambda}^{k}$ is at most $\tbinom{\lambda+k}{k}$.
Note that $\log\left(\left|V_{\lambda}^{k}\right|\right)=O\left(\min\left(k,\lambda\right)\ln\left(\lambda+k\right)\right)$
whereas $\log\left(\left|L^{k}\right|\right)=O\left(k\ln\left(\lambda\right)\right)$.
\end{lem}

The proof of Lemma \ref{lem:size-of-V} is provided in the Appendix.

We define the rounding function $R:\Delta_{k}\to V_{\lambda}^{k}$,
which maps a prediction vector to the corresponding level set in $V_{\lambda}^{k}$:
$R(u)_{i}=\left\lfloor u_{i}\lambda\right\rfloor /\lambda\ \forall i$.
Conversely, we define the function $\rho$ that maps a level set $v\in V_{\lambda}^{k}$
to the closest canonical distribution $\rho(v)=\arg\min_{u\in\Delta_{k},R(u)=v}\left\Vert u-v\right\Vert _{\infty}$.
Finally, we define the projection function $\pi:[0,1]^{k}\to\Delta_{k}$
in $\ell_{2}$ norm : $\pi(v)=\arg\min_{u\in\Delta_{k}}\left\Vert u-v\right\Vert _{2}$.

In some cases, we abuse notation by writing $f(S)$ to denote the
common value of a function $f(x)$ for all $x\in S$, when $f(x)=f(y)$
for all $x,y\in S$ .

For our sample complexity results, we use the following lemmas for
adaptive data analysis and concentration of measure.
\begin{lem}
\cite[Theorem 23]{JungLN0SS20}\label{lem:adaptive-data-analysis}
Let $A$ be an algorithm that, having access to a dataset $S=\{x_{i}\}_{i\in[n]}$,
interactively takes as input a stream of queries $q_{1},\ldots,q_{t}:\mathcal{\mathcal{X}\mathrm{\to[0,1]}}$and
provides a stream of answers $a_{1},\ldots,a_{t}\in[0,1]$. Suppose
that $A$ is $(\varepsilon,0)$-differentially private and that
\[
\P\left[\max_{j\in[t]}|\frac{1}{n}\sum_{i\in[n]}q_{j}(x_{i})-a_{j}|\geq\alpha\right]\leq\beta.
\]
 Then, for any $\eta>0$, 
\[
\P\left[\max_{j\in[t]}\left|\E_{x\sim P}\left[q_{j}(x)\right]-a_{j}\right|\geq\alpha+e^{\varepsilon}-1+\sqrt{\frac{2\ln(2/\eta)}{n}}\right]\leq\beta+\eta.
\]
\end{lem}

\begin{lem}
\cite[Theorem 3.6]{ChungL06}\label{lem:conc-bound} Suppose $X_{1},\ldots,X_{n}$
are independent random variables with $X_{i}\le M$ for all $i$.
Let $X=\sum_{i=1}^{n}X_{i}$ and $\left\Vert X\right\Vert =\sqrt{\sum_{i=1}^{n}\E[X_{i}^{2}]}$.
Then,

\[
\P\left[X\ge\E[X]+\lambda\right]\le\exp\left(-\frac{\lambda^{2}}{2\left(\left\Vert X\right\Vert ^{2}+M\lambda/3\right)}\right).
\]
\end{lem}

\section{Multiclass Calibration under $\ell_{p}$ Error}

\label{sec:Multiclass-Calibration}

In this work, we consider a generalization of the expected calibration
error to arbitrary $\ell_{p}$ norms.
\begin{defn}
Fix $p\geq1$ and $k,\lambda\in\N^{+}$. Consider a $k$-class predictor
$f:\mathcal{X}\to\Delta_{k}$ and a data distribution $D$ over features
$\mathcal{X}$ and labels $[k]$. The $\ell_{p}$ calibration error
of $f$ is defined as
\[
\textrm{Err}_{p}(f)\coloneqq\left(\sum_{v\in V_{\lambda}^{k}}\sum_{j=1}^{k}\left(\textrm{Err}(f,v,j)\right)^{p}\right)^{1/p},
\]
where $V_{\lambda}^{k}$ denotes the set of discretized bins,
\begin{align*}
\textrm{Err}(f,v,j) & \coloneqq\left|\E_{(x,y)\sim D}\left[\left(f(x)_{j}-y_{j}\right)\cdot\one\left[R(f(x))=v\right]\right]\right|\\
 & =\left|\E_{(x,y)\sim D}\left[f(x)_{j}-y_{j}\mid R(f(x))=v\right]\right|\P\left[R(f(x))=v\right]
\end{align*}
 measures the calibration error for bin $v$ and class $j$, and $y$
is the one-hot encoding of the label.
\end{defn}

The special case when $p=1$ corresponds to the expected calibration
error (ECE), while the case when $p\to\infty$ corresponds to the
calibration error considered in \cite{HaghtalabJ023,DworkLLT23}.

Our main result is a new algorithm that calibrates a given predictor
$f$ to achieve $\ell_{p}$ calibration error of at most $\varepsilon$,
using a polynomial number of samples for any $p>1$. Furthermore,
for $p=\infty$, the dependence of the algorithm's sample complexity
on $\varepsilon$ is only $1/\varepsilon^{2}$ up to logarithmic factors,
which is nearly optimal. The squared error of the calibrated predictor
is lower than that of the original predictor, up to a small additive
term introduced by discretization. Up to logarithmic factors, this
additive term due to discretization is similar to the term in the
previous work for binary predictors \cite{Hebert-JohnsonK18}.
\begin{thm}
\label{thm:main}Fix $p>1$, $\varepsilon,\delta\in(0,1)$ and $k\in\N^{+}$.
There exists an algorithm that takes as input a $k$-class predictor
$f:\mathcal{X}\to\Delta_{k}$, and with probability at least $1-\delta$
terminates after $O\left(\frac{2^{2/(p-1)}}{\varepsilon^{2p/(p-1)}}\right)$
time steps with total time polynomial in $k$ and $\frac{1}{\varepsilon}$.
Using
\begin{align*}
 & O\left(\left(\frac{2^{1/(p-1)}}{\varepsilon^{p/(p-1)}}\right)^{2}\log^{3}\left(\frac{2^{1/(p-1)}}{\varepsilon^{p/(p-1)}}\right)\log\left(\frac{2^{1/(p-1)}k}{\varepsilon^{p/(p-1)}\delta}\right)\right)
\end{align*}
 samples from distribution $D$, returns a $k$-class predictor $h:\mathcal{X}\to\Delta_{k}$
that has calibration error $\textrm{Err}_{p}(h)\leq\varepsilon$ and
squared error 
\[
\E_{D}\left[\left\Vert h(x)-y\right\Vert _{2}^{2}\right]-\E_{D}\left[\left\Vert f\left(x\right)-y\right\Vert _{2}^{2}\right]\leq O\left(\frac{\varepsilon^{p/(p-1)}}{2^{1/(p-1)}}\log\left(\frac{2^{1/(p-1)}}{\varepsilon^{p/(p-1)}}\right)\right).
\]
\end{thm}

We present Algorithm \ref{alg:multiclass_cal_full} for calibrating
a given $k$-class predictor $f$. The high-level structure of the
algorithm, outlined in \ref{alg:multiclass_cal}, follows a standard
approach in the literature. It first assigns datapoints to bins based
on the level set of their rounded prediction $f(x)$ , and then iteratively
identifies groups of bins and classes with large calibration error,
applying corrective updates as needed. At each time step $t$, to
correct the prediction for a group of bins $S^{(t)}$ and class $j^{(t)}$
with large calibration error, the algorithm estimates the probability
that datapoints in bins $S^{(t)}$ have label $j^{(t)}$. It then
uses this estimate to correct the prediction vector for $S^{(t)}$
and projects the corrected vector onto the probability simplex $\Delta_{k}$
to ensure valid probability outputs, using this as the new prediction
for datapoints assigned to $S^{(t)}$. If at time step $t$, there
exists another group of bins $S'$ with in the same level set as $S^{(t)}$,
the algorithm merges these two groups. It assigns a single prediction
vector to all the inputs in $S^{(t)}\cup S'$, selecting the prediction
from whichever group has the largest estimated probability mass. However,
merging bins may cause the estimation errors to accumulate, potentially
leading to large calibration errors in the merged group. To mitigate
this, the algorithm re-estimates the calibration error of each group
after merging.

Our algorithm differs from existing binning-based calibration algorithms
in two ways. First, it identifies a set of bins $B$ with large probability
mass, because only such bins can contribute significantly to the overall
calibration error. The algorithm maintains a data structure $G$ containing
disjoint groups of bins that may have large error and iteratively
searches through them to identify groups requiring correction. Initially,
$G$ contains a group for each high-probability bin. As the algorithm
merges groups of bins during the correction process, it updates $G$
accordingly. Second, the algorithm reduces the number of samples needed
to estimate the calibration error by leveraging the fact that groups
of bins are only merged over time and never split, and by applying
Lemma \ref{lem:adaptive-data-analysis} for adaptive data analysis.
The groups of bins $S^{(t)}$are selected adaptively, as their error
depends on the current predictions. If we were to analyze the sample
complexity using standard concentration inequalities, this adaptivity
would require the use of fresh samples at every time step. To avoid
this inefficiency, our algorithm maintains error estimates for $O(\log|B|)$
collections of evolving disjoint groups of bins, denoted collectively
as $M$. Note that $M$ forms a partition of $B$. An interesting
property of this structure is that any group of bins in $G$ for which
we need to estimate the calibration error can be expressed as a disjoint
union of groups in $M$. As a result, the calibration error estimate
of $S^{(t)}$ can be computed efficiently by summing the estimates
for groups in $M$ that are subsets of $S^{(t)}$. The sizes of the
groups in $M$ are powers of $2$ and all groups of the same size
that arise during the execution of the algorithm remain disjoint.
For each group size $2^{i}$ and each type of estimate, we maintain
a separate pool of samples. Since a group in $M$ can contain at most
$|B|$ distinct bins, we need $O(\log|B|)$ separate sample pools.
We analyze the sample complexity after proving Lemma \ref{lem:adaptive-samples},
which bounds the number of samples required to estimate a collection
of disjoint, adaptively chosen queries.

\begin{algorithm}
\caption{\label{alg:multiclass_cal}Multiclass Calibration Outline}

\textbf{Input: }predictor $f$

$\quad$

Discretize prediction space into bins and identify high-probability
bins

Create two parallel data structures:

1. Estimation structure $M$ tracks statistics for groups of bins

2. Prediction structure $G$ stores predictions and tracks calibration
errors per group of bins

Initialize both structures, $M$ and $G$, to contain one group per
high-probability bin

$t\gets0$

While there exists a group of bins in $G$ with large error for some
class $j\in[k]$:

$\quad$Select group $S^{(t)}\in G$ and class $j^{(t)}\in[k]$ with
large error

$\quad$Correct the prediction for $S^{(t)}$and $j^{(t)}$

$\quad$Merge groups in $G$ with similar predictions to that of $S^{(t)}$

$\quad$Update structure $M$

$\quad$Estimate statistics and error for $S^{(t)}$

$\quad$$t\gets t+1$.

Define $h(x)=\begin{cases}
\text{prediction for group \ensuremath{S} in \ensuremath{G} that contains \ensuremath{f(x)} } & \text{if }f(x)\text{ is in a high-probability bin}\\
\text{\text{nearest valid probability vector to }\ensuremath{f(x)}} & \text{o.w.}
\end{cases}$

$\quad$

\textbf{Output: }calibrated predictor $h$
\end{algorithm}
We show that Algorithm\ref{alg:multiclass_cal_full} satisfies Theorem
\ref{thm:main}. The proof is presented step by step in the following
three subsections, with key results organized into several lemmas.
Lemmas \ref{lem:samples-probabilities} and \ref{lem:adaptive-samples}
show that all estimated quantities are within small additive error
of the true quantities. Lemmas \ref{lem:sq-error-stage2}, \ref{lem:sq-error-stage1},
and \ref{lem:sq-error-t=00003D0} provide a bound on the squared error
of the modified predictor. Lemma \ref{lem:algorithm-termination}
proves that the algorithm terminates after $O(2^{2/(p-1)}/\varepsilon^{2p/(p-1)})$
iterations, while Lemma \ref{lem:runtime} shows that the total runtime
is polynomial in $1/\varepsilon$ and $k$. Finally, Lemma \ref{lem:termination-error}
establishes that the calibration error of the final predictor when
the algorithm terminates is smaller than $\varepsilon$. All omitted
proofs are provided in the Appendix.

\begin{algorithm}
\caption{\label{alg:multiclass_cal_full}Multiclass Calibration}
\textbf{Input: }predictor $f$, discretization function $R$, parameters
$\varepsilon$ and $\delta$.

$\quad$

Set $\beta\gets\frac{\varepsilon^{p/(p-1)}}{2^{1/(p-1)}}$ and $\lambda\gets\left\lceil \frac{1}{\beta}\right\rceil $

For all bins $v\in\VV$:

$\quad$Estimate probability mass of bin $v$, $\hat{\mu}_{v}\approx\P[R(f(x))=v]$

Select high-probability bins $B\gets\left\{ v:\hat{\mu}_{v}\geq\nicefrac{\beta}{6}\right\} $
// Focus only on these bins for calibration

$\quad$

$M\gets$initialize with one group $\{v\}$ per high-probability bin
$v$ in $B$

$G\gets$ initialize with one group $\{v\}$ per high-probability
bin $v$ in $B$

$t\gets0$

For each group $\{v\}\in M$:

$\quad$Estimate probability $\hat{P}_{\{v\}}\approx\P[R(f(x))\in\{v\}]$

$\quad$Estimate mean label $\hat{E}_{\{v\},j}\approx\E_{(x,y)\sim D}\left[y_{j}\one\left[R(f(x))\in\{v\}\right]\right]$
for all $j\in[k]$

For each group $\{v\}\in G$:

$\quad$pred$(\{v\})\gets\rho(v)$

$\quad$Compute error $\hat{\text{Err}}(\{v\},j)\gets\left|\hat{P}_{\{v\}}\text{pred}(\{v\})_{j}-\hat{E}_{\{v\},j}\right|$
for each class $j\in[k]$

$\quad$

While $\exists$ group $S\in$$G$ with error $\hat{\textrm{Err}}(S,j)>\nicefrac{\beta}{2}$
for some class $j\in[k]$:

$\quad$Select group $S^{(t)}\in G$ and class $j^{(t)}\in[k]$ with
large error$\hat{\textrm{Err}}(S^{(t)},j^{(t)})>\nicefrac{\beta}{2}$

$\quad$ $z_{j^{(t)}}^{(t)}\gets\min\left(\nicefrac{\left(\sum_{S\in M:S\subseteq S^{(t)}}\hat{E}_{S,j^{(t)}}\right)}{\left(\sum_{S\in M:S\subseteq S^{(t)}}\hat{P}_{S}\right)},1\right)$

$\quad$For all other classes $j$:

$\qquad$$z_{j}^{(t)}\gets$$\text{pred}\left(S^{(t)}\right)_{j}$

$\quad$$\text{pred}\left(S^{(t)}\right)\gets\pi\left(z^{(t)}\right)$

$\quad$If there exists group $S'\neq S^{(t)}$ in $G$ such that
$R\left(\text{pred}\left(S'\right)\right)=R\left(\text{pred}\left(S^{(t)}\right)\right)$:

$\qquad$Merge $S^{(t)}$and $S'$ into a single group in $G$

$\qquad$If $\sum_{S\in M:S\subseteq S^{(t)}}\hat{P}_{S}\leq\sum_{S\in M:S\subseteq S'}\hat{P}_{S}$:

$\qquad\quad\text{pred}\left(S^{(t)}\cup S'\right)\gets\text{pred}\left(S'\right)$

$\qquad$else:

$\qquad\quad\text{pred}\left(S^{(t)}\cup S'\right)\gets\text{pred}\left(S^{(t)}\right)$

$\qquad S^{(t)}\gets S^{(t)}\cup S'$

$\quad$While there exist groups $S_{1}\neq S_{2}$ in $M$ that are
subsets of $S^{(t)}$ with the same cardinality:

$\qquad$Merge $S_{1}$ and $S_{2}$ in $M$

$\qquad$Estimate probability $\hat{P}_{S_{1}\cup S_{2}}\approx\P[R(f(x))\in S_{1}\cup S_{2}]$

$\qquad$Estimate mean label $\hat{E}_{S_{1}\cup S_{2},j}\approx\E_{(x,y)\sim D}\left[y_{j}\one\left[R(f(x))\in S_{1}\cup S_{2}\right]\right]$
for all $j\in[k]$

$\quad$Compute error $\hat{\text{Err}}(S^{(t)},j)\gets\left|\left(\sum_{S\in M:S\subseteq S^{(t)}}\hat{P}_{S}\right)\text{pred}\left(S^{(t)}\right)_{j}-\sum_{S\in M:S\subseteq S^{(t)}}\hat{E}_{S,j}\right|,$
$\forall j\in[k]$

$\quad$$t\gets t+1$.

$\quad$

Define $h(x)=\begin{cases}
\text{\text{pred}\ensuremath{(S)},\text{where }\ensuremath{S}\text{ is the group in \ensuremath{G} that contains \ensuremath{R\left(f(x)\right)}} } & \text{if }R\left(f(x)\right)\in B\\
\text{\ensuremath{\rho\left(R\left(f(x)\right)\right)}} & \text{o.w.}
\end{cases}$

$\quad$

\textbf{Output:} $h$
\end{algorithm}

\subsection{Correctness of Estimates}

\label{subsec:Correctness-of-Estimates}

In Algorithm \ref{alg:multiclass_cal_full} we use samples to compute
three types of estimates. For the algorithm to function correctly,
the estimates need to be sufficiently accurate. This requirement is
captured by the following three events.

\textbf{Important Events:}
\begin{enumerate}
\item Event $A_{1}$: $\left|\hat{\mu}_{v}-\P\left[R(f(x))=v\right]\right|\leq\frac{\beta}{12},\;\forall v\in\VV$.
\item Event $A_{2}$: $\left|\hat{P}_{S}-\P\left[R(f(x))\in S\right]\right|\leq\frac{\beta}{36(\left\lfloor \log_{2}|B|\right\rfloor +1)}$,
for all groups of bins $S$ in $M$ that ever occur during the execution
of the algorithm.
\item Event $A_{3}$: $\left|\hat{E}_{S,j}-\E_{(x,y)\sim D}\left[y_{j}\one\left[R(f(x))\in S\right]\right]\right|\leq\frac{\beta}{36(\left\lfloor \log_{2}|B|\right\rfloor +1)}$,
for all groups of bins $S$ in $M$ that ever occur during the execution
of the algorithm and all classes $j\in[k]$.
\end{enumerate}
First, for every level set $v\in\VV$ we estimate the probability
that the rounded prediction of the given predictor $R(f(x))$ equals
$v$. By Lemma \ref{lem:samples-probabilities}, if we set $\alpha_{1}=\nicefrac{\beta}{12}$
and $\delta_{1}=\nicefrac{\delta}{3}$, we know that using $O\left(\frac{1}{\beta}\log\left(\frac{\left|V_{\lambda}^{k}\right|}{\delta}\right)+\frac{1}{\beta^{2}}\log\left(\frac{1}{\beta\delta}\right)\right)$
we get estimates such that with probability at least $1-\nicefrac{\delta}{3}$

\[
\left|\hat{\mu}_{v}-\P\left[R(f(x))=v\right]\right|\leq\frac{\beta}{12},\;\forall v\in\VV.
\]

\begin{lem}
\label{lem:samples-probabilities}Fix $\delta_{1},\alpha_{1}\in(0,1)$.
Using $O\left(\frac{1}{\alpha_{1}}\log\left(\frac{\left|V_{\lambda}^{k}\right|}{\delta_{1}}\right)+\frac{1}{\alpha_{1}^{2}}\log\left(\frac{1}{\alpha_{1}\delta_{1}}\right)\right)$samples,
we can estimate $\hat{\mu}_{v}$, for all $v\in\VV$, s.t. with probability
at least $1-\delta_{1}$

\[
\left|\hat{\mu}_{v}-\P\left[R(f(x))=v\right]\right|\leq\alpha_{1},\;\forall v\in\VV.
\]
\end{lem}

For every group of bins $S$ that appears in $M$ during the execution
of the algorithm, we estimate two types of quantities: the probability
that the prediction $R(f(x))$ is in one of the bins in $S$ and the
expected label $y_{j}$ of points $(x,y)$ whose prediction $R(f(x))$
is in one of the bins in $S$, for all $j\in[k]$. The sizes of groups
in $M$ are all powers of $2$ and all groups of the same size that
occur during the execution of the algorithm are disjoint. For each
group size $2^{i}$ and for each type of estimate, probability or
expected label, we maintain a separate pool of samples. Since there
are can be at most $|B|$ distinct bins in a group in $M$, we need
$O(\log|B|)$ separate sample pools. To analyze the sample complexity,
we apply the adaptive data analysis result of Lemma \ref{lem:adaptive-samples}
because the algorithm picks the set that needs adjustment adaptively
at each time step.
\begin{lem}
\label{lem:adaptive-samples}Fix $n,k\in\N^{+}$and $\alpha,\delta\in(0,1)$.
Consider an adaptive algorithm $A$, a distribution $D$ over the
domain $\mathcal{X\times\mathcal{Y}}$, and a function $\phi:\mathcal{X\times\mathcal{Y}}\to\Delta_{k}$.
The algorithm adaptively selects a sequence of $n$ disjoint events
for $D$ as follows. First, it selects $E_{1}$ and estimates $\E_{(x,y)\sim D}\left[\phi(x,y)_{j}\cdot\one\left[(x,y)\in E_{1}\right]\right]$,
for all $j\in[k]$. Then, it selects event $E_{2}$, disjoint from
$E_{1}$, and estimates $\E_{(x,y)\sim D}\left[\phi(x,y)_{j}\cdot\one\left[(x,y)\in E_{2}\right]\right]$,
for all $j\in[k]$, and so on. With $O\left(\frac{\log\left(\nicefrac{nk}{\delta}\right)}{\alpha^{2}}\right)$
shared samples, we can estimate all expectations up to additive error
$\alpha$ and failure probability $\delta$.
\end{lem}

By Lemma \ref{lem:adaptive-samples} we get that for a fixed group
size $2^{i}\leq|B|$, using $O\left(\frac{\log^{2}\left(|B|\right)\log\left(\nicefrac{|B|\log|B|}{\delta}\right)}{\beta^{2}}\right)$
samples we get probability estimates such that with probability at
least $1-\frac{\delta}{3\left(\left\lfloor \log_{2}|B|\right\rfloor +1\right)}$

\[
\left|\hat{P}_{S}-\P\left[R(f(x))\in S\right]\right|\leq\frac{\beta}{36(\left\lfloor \log_{2}|B|\right\rfloor +1)},
\]
for all groups of bins $S$ in $M$ of size $2^{i}$ that ever occur
during the execution of the algorithm. Similarly, by Lemma \ref{lem:adaptive-samples}
we get that for a fixed group size $2^{i}\leq|B|$, using $O\left(\frac{\log^{2}\left(|B|\right)\log\left(\nicefrac{|B|k\log|B|}{\delta}\right)}{\beta^{2}}\right)$
samples we get expected label estimates such that with probability
at least $1-\frac{\delta}{3\left(\left\lfloor \log_{2}|B|\right\rfloor +1\right)}$

\[
\left|\hat{E}_{S,j}-\E_{(x,y)\sim D}\left[y_{j}\one\left[R(f(x))\in S\right]\right]\right|\leq\frac{\beta}{36(\left\lfloor \log_{2}|B|\right\rfloor +1)},
\]
for all groups of bins $S$ in $M$ of size $2^{i}$ that ever occur
during the execution of the algorithm and all classes $j\in[k]$.

The number of groups with different sizes up to $|B|$ that are powers
of $2$ is at most $\left\lfloor \log_{2}|B|\right\rfloor +1$. Thus,
we have that 
\begin{align*}
 & \P\left[\neg A_{1}\textrm{ or }\neg A_{2}\textrm{or }\neg A_{3}\right]\\
 & \leq\P[\neg A_{1}]+\P[\neg A_{2}]+\P[\neg A_{3}]\\
 & \leq\frac{\delta}{3}+\left(\left\lfloor \log_{2}|B|\right\rfloor +1\right)\frac{\delta}{3\left(\left\lfloor \log_{2}|B|\right\rfloor +1\right)}+\left(\left\lfloor \log_{2}|B|\right\rfloor +1\right)\frac{\delta}{3\left(\left\lfloor \log_{2}|B|\right\rfloor +1\right)}\leq\delta
\end{align*}

If event $A_{1}$ is true, then the size of $|B|$ is at most $O\left(\frac{1}{\beta}\right)$
because $B=\left\{ v:v\in V_{\lambda}^{k},\hat{\mu}_{v}\geq\nicefrac{\beta}{6}\right\} $
and $\sum_{v\in\VV}\P\left[R\left(f\left(x\right)\right)=v\right]=1$.
Thus, the algorithm uses $O\left(\frac{1}{\beta}\log\left(\frac{\left|V_{\lambda}^{k}\right|}{\delta}\right)+\frac{1}{\beta^{2}}\log^{3}\left(\frac{1}{\beta}\right)\log\left(\frac{k\log\left(\nicefrac{1}{\beta}\right)}{\beta\delta}\right)\right)$
samples in total. Lemma \ref{lem:size-of-V} provides a bound on the
size of $\VV$.

To estimate the probability of a group of bins $S\in G$, we compute
the sum of probability estimates for all subsets $S'\subseteq S$
that are in $M$ and use the following lemma to bound the error. We
estimate the expected label in a similar way.
\begin{lem}
\label{lem:size-of-M}For each $S\in G$, the number of subsets $S'\in M$
such that $S'\subseteq S$ is at most $O\left(\log|B|\right)$.
\end{lem}

\subsection{Accuracy of the Calibrated Predictor}

\label{subsec:Accuracy}

In this subsection, we show that if the estimates are accurate, then
Algorithm \ref{alg:multiclass_cal_full} constructs a multiclass predictor
whose squared error is lower than that of the given predictor, up
to a small additive term introduced by discretization. At each round
$t$ before the algorithm terminates, it selects a bin $S^{(t)}$
and a coordinate $j^{(t)}$ with high calibration error. The algorithm
then updates the predictor in two stages. In Stage 1, it computes
an improved prediction vector $z^{(t)}$ for the selected bin and
projects it to the simplex to obtain $\text{pred}\left(S^{(t)}\right)$.
In Stage 2, it checks if there is another group $S'$ that gets mapped
by to the same level set as $S^{(t)}$, if so it merges $S'$ and
$S^{(t)}$. We analyze the change in the squared error at each time
step by examining separately the change due to Stage 1 and Stage 2.
Notably, in Lemma \ref{lem:sq-error-stage1} we show that the squared
error always decreases in Stage 1, whereas in Lemma \ref{lem:sq-error-stage2}
we demonstrate that Stage 2 might lead to a small increase. In both
Lemmas, we assume that the all the estimated quantities are accurate,
meaning that events $A_{1},A_{2}$ and $A_{3}$ as defined in the
previous subsection hold. Lemma \ref{lem:sq-error-t=00003D0} provides
an upper bound on the squared error due to the discretization of $f$.

For the purposes of this proof we define 
\[
h_{t}(x)=\begin{cases}
\text{\text{pred}\ensuremath{(S)},\text{ where }\ensuremath{S} in \ensuremath{G} that contains \ensuremath{R\left(f(x)\right)}at time step \ensuremath{t} } & \text{if }R\left(f(x)\right)\in B\\
\text{\ensuremath{\rho\left(R\left(f(x)\right)\right)}} & \text{o.w.}
\end{cases}
\]

\begin{lem}
\label{lem:sq-error-stage2} Assuming that $A_{1},A_{2}$ and $A_{3}$
hold, after $T$ time steps of the algorithm, the squared error of
the predictor $h$ is 
\begin{align*}
 & \E\left[\left\Vert h\left(x\right)-y\right\Vert _{2}^{2}\right]\\
 & \leq\E\left[\left\Vert h_{0}(x)-y\right\Vert _{2}^{2}\right]\\
 & \quad+\sum_{t=0}^{T-1}\E\left[\left\Vert \pi\left(z^{(t)}\right)-y\right\Vert _{2}^{2}-\left\Vert h_{t}\left(x\right)-y\right\Vert _{2}^{2}\left|R\left(f\left(x\right)\right)\in S^{(t)}\right.\right]\P\left[R\left(f\left(x\right)\right)\in S^{(t)}\right]\\
 & \quad+O\left(\beta\log\left(\frac{1}{\beta}\right)\right).
\end{align*}
\end{lem}

\begin{lem}
\label{lem:sq-error-stage1} Assuming that $A_{1},A_{2}$ and $A_{3}$
hold, at time step $t$ of the algorithm
\[
\E\left[\left\Vert \pi\left(z^{(t)}\right)-y\right\Vert _{2}^{2}-\left\Vert h_{t}\left(x\right)-y\right\Vert _{2}^{2}\left|R\left(f\left(x\right)\right)\in S^{(t)}\right.\right]\P\left[R\left(f\left(x\right)\right)\in S^{(t)}\right]\leq-\nicefrac{\beta^{2}}{9}.
\]
\end{lem}

\begin{lem}
\label{lem:sq-error-t=00003D0}The squared error at time step $0$
is
\[
\E\left[\left\Vert h_{0}(x)-y\right\Vert _{2}^{2}\right]\leq\E\left[\left\Vert f\left(x\right)-y\right\Vert _{2}^{2}\right]+O\left(\beta\right).
\]
\end{lem}

\subsection{Termination of the Algorithm with Small Calibration Error}

\label{subsec:Termination}

In this subsection, we show that, assuming that the estimates are
accurate, the algorithm terminates after $O\left(\nicefrac{1}{\beta^{2}}\right)$
steps with $\ell_{p}$ calibration error at most $O\left(\beta^{\nicefrac{p-1}{p}}\right)$.
Moreover, its total runtime is polynomial in $\nicefrac{1}{\beta}$
and $k$.
\begin{lem}
\label{lem:algorithm-termination}Assuming that $A_{1}$, $A_{2}$,
and $A_{3}$ hold, the algorithm terminates after at most $O\left(\nicefrac{1}{\beta^{2}}\right)$
time steps.
\end{lem}

\begin{lem}
\label{lem:termination-error}Assuming that $A_{1}$, $A_{2}$, and
$A_{3}$ hold, the $\ell_{p}$ calibration error of $h$ $\left(\textrm{Err}_{p}(h)\right)^{p}$
is bounded by $O(\beta^{p-1})$.
\end{lem}

\begin{lem}
\label{lem:runtime}Assuming that $A_{1}$, $A_{2}$, and $A_{3}$
hold, the algorithm terminates in time polynomial in $\frac{1}{\beta}$
and $k$.
\end{lem}

Combining the results of Subsections \ref{subsec:Correctness-of-Estimates},
\ref{subsec:Accuracy}, and \ref{subsec:Termination}, we obtain the
proof of Theorem \ref{thm:main}.

\section{Conclusion}

In this work, we introduced the $\ell_{p}$ calibration error for
multiclass predictors and presented an algorithm that modifies a given
predictor to achieve low calibration error while preserving its accuracy
using only a polynomial number of samples in the number of classes.
The algorithm can be applied to any value of $p>1$ and improves the
known sample complexity in the case of $p=\infty$.

Related work in this area has explored multicalibration, where the
calibration guarantees hold for many, possibly overlapping, populations.
While our work focuses on calibration, an interesting direction for
future research is to generalize our results to obtain stronger sample
complexity in that setting as well.

\bibliographystyle{plain}
\bibliography{calibration,C:/Users/brkko/Downloads/KumarLM19,C:/Users/brkko/Downloads/VaicenaviciusWA19,C:/Users/brkko/Downloads/WidmannLZ19,C:/Users/brkko/Downloads/ZhangKH20,C:/Users/brkko/Downloads/PopordanoskaSB22,C:/Users/brkko/Downloads/Gruber022}

\appendix

\section{Proofs from Section \ref{sec:Preliminaries}}
\begin{lem}[Lemma \ref{lem:size-of-V} restated]
For any $\lambda,k\in\N^{+}$, the number of level sets in $V_{\lambda}^{k}$
is at most $\tbinom{\lambda+k}{k}$. Note that $\log\left(\left|V_{\lambda}^{k}\right|\right)=O\left(\min\left(k,\lambda\right)\ln\left(k+\lambda\right)\right)$
whereas $\log\left(\left|L^{k}\right|\right)=O\left(k\ln\left(\lambda\right)\right)$.
\end{lem}

\begin{proof}
Every $v\in\VV$ corresponds to a $u\in\Delta_{k}$. Therefore, we
have that 
\begin{align*}
\sum_{i\in[k]}v_{i} & =\sum_{i\in[k]}\frac{\left\lfloor u_{i}\lambda\right\rfloor }{\lambda}=1-\left(1-\sum_{i\in[k]}\frac{\left\lfloor u_{i}\lambda\right\rfloor }{\lambda}\right).
\end{align*}

Let $v_{k+1}=1-\sum_{i\in[k]}\frac{\left\lfloor u_{i}\lambda\right\rfloor }{\lambda}$,
which is a non-negative integer multiple of $\nicefrac{1}{\lambda}$.
By rearranging the terms, we have that $\sum_{i\in[k+1]}v_{i}=1$.
The number of $k+1$tuples of non-negative integer multiples of $\nicefrac{1}{\lambda}$that
sum up to $1$is $\tbinom{\lambda+k}{k}$. Therefore, $\left|\VV\right|=\tbinom{\lambda+k}{k}$.
\end{proof}

\section{Proofs from Subsection \ref{subsec:Correctness-of-Estimates}}
\begin{lem}[Lemma \ref{lem:samples-probabilities} restated]
 Fix $\delta_{1},\alpha_{1}\in(0,1)$. Using $O\left(\frac{1}{\alpha_{1}}\log\left(\frac{\left|V_{\lambda}^{k}\right|}{\delta_{1}}\right)+\frac{1}{\alpha_{1}^{2}}\log\left(\frac{1}{\alpha_{1}\delta_{1}}\right)\right)$samples,
we can estimate $\hat{\mu_{v}}$, for all $v\in\VV$, s.t. with probability
at least $1-\delta_{1}$

\[
\left|\hat{\mu}_{v}-\P\left[R(f(x))=v\right]\right|\leq\alpha_{1},\;\forall v\in\VV.
\]
\end{lem}

\begin{proof}
There are at most $\frac{1}{\alpha_{1}}$bins such that $\P[R(f(x))=v]\ge\alpha_{1}$.
We show that using $m_{1}=\frac{1}{2\alpha_{1}^{2}}\ln\left(\frac{4}{\alpha_{1}\delta_{1}}\right)$
samples, we can estimate all of them up to additive error $\alpha_{1}$.
By applying the Hoeffding inequality and a union bound we obtain that

\begin{align*}
 & \P\left[\exists v\textrm{ s.t. }\P[R(f(x))=v]\ge\alpha_{1}:\left|\hat{\mu}_{v}-\P\left[R(f(x))=v\right]\right|\geq\alpha_{1}\right]\\
 & \leq\frac{2|\left\{ v:\P[R(f(x))=v]\ge\alpha_{1}\right\} |}{e^{2\alpha_{1}^{2}m_{1}}}\\
 & \leq\frac{2}{\alpha_{1}e^{2\alpha_{1}^{2}m_{1}}}\leq\frac{\delta_{1}}{2}.
\end{align*}

For the rest of the bins whose probabilities are less than $\alpha_{1}$,
we show that using $m_{2}=\frac{4}{3\alpha_{1}}\ln\left(\nicefrac{2\left|V_{\lambda}^{k}\right|}{\delta_{1}}\right)$
samples is enough to estimate all of them up to additive error $\alpha_{1}$.
In this case, we have that for all $v$ such that $\P\left[R(f(x))=v\right]<\alpha_{1}$,
$\P\left[R(f(x))=v\right]-\hat{\mu}_{v}<\alpha_{1}$. By applying
Lemma \ref{lem:conc-bound} we also get that 
\[
\P\left[\exists v\textrm{ s.t. }\P[R(f(x))=v]<\alpha_{1}:\hat{\mu}_{v}-\P\left[R(f(x))=v\right]\geq\alpha_{1}\right]\le\left|V_{\lambda}^{k}\right|\cdot\exp\left(-\frac{m_{2}\alpha_{1}^{2}}{2\left(\alpha_{1}+\alpha_{1}/3\right)}\right)\le\frac{\delta_{1}}{2}.
\]

By union bound we obtain that if we use $O\left(\frac{1}{\alpha_{1}}\log\left(\frac{\left|V_{\lambda}^{k}\right|}{\delta_{1}}\right)+\frac{1}{\alpha_{1}^{2}}\log\left(\frac{1}{\alpha_{1}\delta_{1}}\right)\right)$
samples, then

$\P\left[\exists v\in\VV:\left|\hat{\mu}_{v}-\P\left[R(f(x))=v\right]\right|\geq\alpha_{1}\right]\leq\delta_{1}$.
\end{proof}
\begin{lem}[Lemma \ref{lem:adaptive-samples} restated]
 Fix $n,k\in\N^{+}$and $\alpha,\delta\in(0,1)$. Consider an adaptive
algorithm $A$, a distribution $D$ over the domain $\mathcal{X\times\mathcal{Y}}$,
and a function $\phi:\mathcal{X\times\mathcal{Y}}\to\Delta_{k}$.
The algorithm adaptively selects a sequence of $n$ disjoint events
for $D$ as follows. First, it selects $E_{1}$ and estimates $\E_{(x,y)\sim D}\left[\phi(x,y)_{j}\cdot\one\left[(x,y)\in E_{1}\right]\right]$,
for all $j\in[k]$. Then, it selects event $E_{2}$, disjoint from
$E_{1}$, and estimates $\E_{(x,y)\sim D}\left[\phi(x,y)_{j}\cdot\one\left[(x,y)\in E_{2}\right]\right]$,
for all $j\in[k]$, and so on. With $O\left(\frac{\log\left(\nicefrac{nk}{\delta}\right)}{\alpha^{2}}\right)$
shared samples, we can estimate all expectations up to additive error
$\alpha$ and failure probability $\delta$.
\end{lem}

\begin{proof}
There are many ways to achieve this. Here, we describe one approach
using differential privacy and a transfer theorem to adaptive analysis.
The algorithm uses a set $S$ of $m=\frac{32\ln(\nicefrac{4nk}{\delta})}{\alpha^{2}}$
samples and for each event $E_{i}$ and coordinate $j\in[k]$, it
reports $\widehat{e}_{i,j}=\frac{1}{m}\sum_{u\in S}\phi(u)_{j}\cdot\one\left[u\in E_{i}\right]+\varepsilon_{i,j}$,
where $\varepsilon_{i,j}\sim\textrm{Lap}(8/(m\alpha))$. Because the
events are disjoint and each sample contributes to at most one event,
the $\ell_{1}$ global sensitivity of the $k\times n$-dimensional
vector $(e_{1,1},\ldots,e_{1,k},\ldots,e_{n,1},\ldots,e_{n,k})$,
where $e_{i,j}=\frac{1}{m}\sum_{u\in S}\phi(u)_{j}\cdot\one\left[u\in E_{i}\right]$,
is at most $2/m$. Hence, algorithm $A$ is $(\alpha/4,0)$-differentially
private. Since $\varepsilon_{1,1},\ldots,\varepsilon_{n,k}$ are i.i.d.
Laplace random variables with $\lambda=\frac{8}{m\alpha}$, we know
that for any $t>0$, $\P\left[\max_{i\in[n],j\in[k]}|\varepsilon_{i,j}|>t\lambda\right]\leq nde^{-t}$.
For $t=\ln(\nicefrac{2nk}{\delta})$, we get that with probability
at least $1-\frac{\delta}{2}$, the maximum additive error $|\varepsilon_{i,j}|$
is at most $\frac{8\ln(\nicefrac{2nk}{\delta})}{m\alpha}$. By Lemma
\ref{lem:adaptive-data-analysis}, with probability at least $1-\delta$,
we have that

\begin{align*}
 & \max_{i\in[n],j\in[d]}\left|\E_{(x,y)\sim D}\left[\phi(x,y)_{j}\cdot\one\left[(x,y)\in E_{i}\right]\right]-\widehat{e}_{i,j}\right|\le\frac{8\ln\left(\frac{2nk}{\delta}\right)}{m\alpha}+e^{\alpha/4}-1+\sqrt{\frac{2\ln\left(\frac{4}{\delta}\right)}{m}}\\
 & \le\frac{\alpha}{4}+\frac{\alpha}{2}+\frac{\alpha}{4}=\alpha.
\end{align*}
\end{proof}
\begin{lem}[Lemma \ref{lem:size-of-M} restated]
 For each $S\in G$, the number of subsets $S'\in M$ such that $S'\subseteq S$
is at most $O\left(\log|B|\right)$.
\end{lem}

\begin{proof}
For a fixed $S\in G$, all $S'\in M$ such that $S'\subseteq S$ are
of different sizes. This holds because if there were two subsets $S_{1},S_{2}\in M$
such that $S_{1},S_{2}\subseteq S$ and $|S_{1}|=|S_{2}|$, we would
have already merged them. Additionally, the sizes of all $S'\in M$
are powers of $2$. The number of sets with different sizes up to
$|B|$ that are powers of $2$ is at most $\left\lfloor \log_{2}|B|\right\rfloor +1$.
\end{proof}

\section{Proofs from Subsection \ref{subsec:Accuracy}}
\begin{lem}[Lemma \ref{lem:sq-error-stage2} restated]
Assuming that $A_{1},A_{2}$ and $A_{3}$ hold, after $T$ time steps
of the algorithm, the squared error of the predictor $h$ is 
\begin{align*}
 & \E\left[\left\Vert h(x)-y\right\Vert _{2}^{2}\right]\\
 & \leq\E\left[\left\Vert \rho\left(R\left(f\left(x\right)\right)\right)-y\right\Vert _{2}^{2}\right]\\
 & \quad+\sum_{t=0}^{T-1}\E\left[\left\Vert \pi\left(z^{(t)}\right)-y\right\Vert _{2}^{2}-\left\Vert h_{t}\left(x\right)-y\right\Vert _{2}^{2}\left|R\left(f\left(x\right)\right)\in S^{(t)}\right.\right]\P\left[R\left(f\left(x\right)\right)\in S^{(t)}\right]\\
 & \quad+O\left(\beta\log\left(\frac{1}{\beta}\right)\right).
\end{align*}
\end{lem}

\begin{proof}
At each time step $t\leq T-1$ there are three possible cases depending
on whether and how the algorithm merges bins after updating the prediction
for $S^{(t)}$.

Case 1: there is no $S'$ such that $R\left(\pi\left(z^{(t)}\right)\right)=R\left(\text{pred}\left(S'\right)\right)$.
Then, 

\begin{align*}
 & \E\left[\left\Vert h_{t+1}\left(x\right)-y\right\Vert _{2}^{2}\right]-\E\left[\left\Vert h_{t}\left(x\right)-y\right\Vert ^{2}\right]\\
 & =\E\left[\left\Vert h_{t+1}\left(x\right)-y\right\Vert _{2}^{2}-\left\Vert h_{t}\left(x\right)-y\right\Vert _{2}^{2}\left|R\left(f\left(x\right)\right)\in S^{(t)}\right.\right]\P\left[R\left(f\left(x\right)\right)\in S^{(t)}\right]\\
 & =\E\left[\left\Vert \text{\ensuremath{\pi}\ensuremath{\left(z^{(t)}\right)}}-y\right\Vert _{2}^{2}-\left\Vert h_{t}\left(x\right)-y\right\Vert _{2}^{2}\left|R\left(f\left(x\right)\right)\in S^{(t)}\right.\right]\P\left[R\left(f\left(x\right)\right)\in S^{(t)}\right].
\end{align*}

Case 2: there is a $S'$ such that $R\left(\pi\left(z^{(t)}\right)\right)=R\left(\text{pred}\left(S'\right)\right)$
and $\sum_{S\in M:S\subseteq S^{(t)}}\hat{P}_{S}>\sum_{S\in M:S\subseteq S'}\hat{P}_{S}$.
Then, 
\begin{align*}
 & \E\left[\left\Vert h_{t+1}\left(x\right)-y\right\Vert _{2}^{2}\right]-\E\left[\left\Vert h_{t}\left(x\right)-y\right\Vert _{2}^{2}\right]\\
 & =\E\left[\left\Vert \ensuremath{\pi}\left(z^{(t)}\right)-y\right\Vert _{2}^{2}-\left\Vert h_{t}\left(x\right)-y\right\Vert _{2}^{2}\left|R\left(f\left(x\right)\right)\in S^{(t)}\right.\right]\P\left[R\left(f\left(x\right)\right)\in S^{(t)}\right]\\
 & \quad+\E\left[\left\Vert \ensuremath{\pi}\left(z^{(t)}\right)-y\right\Vert _{2}^{2}-\left\Vert h_{t}\left(x\right)-y\right\Vert _{2}^{2}\left|R\left(f\left(x\right)\right)\in S'\right.\right]\P\left[R\left(f\left(x\right)\right)\in S'\right]\\
 & \leq\E\left[\left\Vert \ensuremath{\pi}\left(z^{(t)}\right)-y\right\Vert _{2}^{2}-\left\Vert h_{t}\left(x\right)-y\right\Vert _{2}^{2}\left|R\left(f\left(x\right)\right)\in S^{(t)}\right.\right]\P\left[R\left(f\left(x\right)\right)\in S^{(t)}\right]\\
 & \quad+\frac{4}{\lambda}\P\left[R\left(f\left(x\right)\right)\in S'\right].
\end{align*}

The last inequality holds because if $R\left(f\left(x\right)\right)\in S'$,
we have that
\begin{align*}
 & \E\left[\left\Vert \ensuremath{\pi}\left(z^{(t)}\right)-y\right\Vert _{2}^{2}-\left\Vert h_{t}\left(x\right)-y\right\Vert _{2}^{2}\left|R\left(f\left(x\right)\right)\in S'\right.\right]\\
 & =\E\left[\left\Vert \ensuremath{\pi}\left(z^{(t)}\right)-y\right\Vert _{2}^{2}-\left\Vert \text{pred}\left(S'\right)-y\right\Vert _{2}^{2}\left|R\left(f\left(x\right)\right)\in S'\right.\right]\\
 & \leq\left\Vert \ensuremath{\pi}\left(z^{(t)}\right)\right\Vert _{2}^{2}-\left\Vert \text{pred}\left(S'\right)\right\Vert _{2}^{2}+2\max_{j\in[k]}\left|\ensuremath{\pi}\left(z^{(t)}\right)_{j}-\text{pred}\left(S'\right)_{j}\right|\\
 & \leq\left(\max_{j\in[k]}\left|\ensuremath{\pi}\left(z^{(t)}\right)_{j}-\text{pred}\left(S'\right)_{j}\right|\right)\sum_{j\in[k]}\left(\left|\ensuremath{\pi}\left(z^{(t)}\right)_{j}\right|+\left|\text{pred}\left(S'\right)_{j}\right|\right)\\
 & \quad+2\max_{j\in[k]}\left|\ensuremath{\pi}\left(z^{(t)}\right)_{j}-\text{pred}\left(S'\right)_{j}\right|.
\end{align*}

Since both $\ensuremath{\pi}\left(z^{(t)}\right)$ and $\text{pred}\left(S'\right)$
are in the same level set when rounded by $R$, for each coordinate
$j\in[k]$, $\left|\ensuremath{\pi}\left(z^{(t)}\right)_{j}-\text{pred}\left(S'\right)_{j}\right|\leq\nicefrac{1}{\lambda}$.
Furthermore, both $\ensuremath{\pi}\left(z^{(t)}\right)$ and $\text{pred}\left(S'\right)$
are probability distributions and, hence, their coordinates sum to
$1$. Therefore, 
\begin{align*}
\left(\max_{j\in[k]}\left|\ensuremath{\pi}\left(z^{(t)}\right)_{j}-\text{pred}\left(S'\right)_{j}\right|\right)\sum_{j\in[k]}\left(\left|\ensuremath{\pi}\left(z^{(t)}\right)_{j}\right|+\left|\text{pred}\left(S'\right)_{j}\right|\right) & \leq\frac{2}{\lambda}.
\end{align*}

Case 3: there is a $S'$ such that $R\left(\pi\left(z^{(t)}\right)\right)=R\left(\text{pred}\left(S'\right)\right)$
and $\sum_{S\in M:S\subseteq S^{(t)}}\hat{P}_{S}\leq\sum_{S\in M:S\subseteq S'}\hat{P}_{S}$.
Then, 
\begin{align*}
 & \E\left[\left\Vert h_{t+1}\left(x\right)-y\right\Vert _{2}^{2}\right]-\E\left[\left\Vert h_{t}\left(x\right)-y\right\Vert _{2}^{2}\right]\\
 & =\E\left[\left\Vert \text{pred}\left(S'\right)-y\right\Vert _{2}^{2}-\left\Vert h_{t}\left(x\right)-y\right\Vert _{2}^{2}\left|R\left(f\left(x\right)\right)\in S^{(t)}\right.\right]\P\left[R\left(f\left(x\right)\right)\in S^{(t)}\right]\\
 & =\E\left[\left\Vert \ensuremath{\pi}\left(z^{(t)}\right)-y\right\Vert _{2}^{2}-\left\Vert h_{t}\left(x\right)-y\right\Vert _{2}^{2}\left|R\left(f\left(x\right)\right)\in S^{(t)}\right.\right]\P\left[R\left(f\left(x\right)\right)\in S^{(t)}\right]\\
 & \quad+\E\left[\left\Vert \text{pred}\left(S'\right)-y\right\Vert _{2}^{2}-\left\Vert \ensuremath{\pi}\left(z^{(t)}\right)-y\right\Vert _{2}^{2}\left|R\left(f\left(x\right)\right)\in S^{(t)}\right.\right]\P\left[R\left(f\left(x\right)\right)\in S^{(t)}\right]\\
 & \leq\E\left[\left\Vert \ensuremath{\pi}\left(z^{(t)}\right)-y\right\Vert _{2}^{2}-\left\Vert h_{t}\left(x\right)-y\right\Vert _{2}^{2}\left|R\left(f\left(x\right)\right)\in S^{(t)}\right.\right]\P\left[R\left(f\left(x\right)\right)\in S^{(t)}\right]\\
 & \quad+\frac{4}{\lambda}\P\left[R\left(f\left(x\right)\right)\in S^{(t)}\right].
\end{align*}

Similary to the previous case, the last inequality holds because we
have that
\begin{align*}
 & \E\left[\left\Vert \text{pred}\left(S'\right)-y\right\Vert _{2}^{2}-\left\Vert \ensuremath{\pi}\left(z^{(t)}\right)-y\right\Vert _{2}^{2}\left|R\left(f\left(x\right)\right)\in S^{(t)}\right.\right]\\
 & \leq\left\Vert \text{pred}\left(S'\right)\right\Vert _{2}^{2}-\left\Vert \ensuremath{\pi}\left(z^{(t)}\right)\right\Vert _{2}^{2}+2\max_{j\in[k]}\left|\ensuremath{\pi}\left(z^{(t)}\right)_{j}-\text{pred}\left(S'\right)_{j}\right|\\
 & \leq\left(\max_{j\in[k]}\left|\text{pred}\left(S'\right)_{j}-\ensuremath{\pi}\left(z^{(t)}\right)_{j}\right|\right)\sum_{j\in[k]}\left(\left|\text{pred}\left(S'\right)_{j}\right|+\left|\ensuremath{\pi}\left(z^{(t)}\right)_{j}\right|\right)\\
 & \quad+2\max_{j\in[k]}\left|\ensuremath{\pi}\left(z^{(t)}\right)_{j}-\text{pred}\left(S'\right)_{j}\right|\\
 & \leq\frac{4}{\lambda}.
\end{align*}

In all three cases discussed above, the upper bound includes the term
\[
\E\left[\left\Vert \ensuremath{\pi}\left(z^{(t)}\right)-y\right\Vert _{2}^{2}-\left\Vert h_{t}\left(x\right)-y\right\Vert _{2}^{2}\left|R\left(f\left(x\right)\right)\in S^{(t)}\right.\right]\P\left[R\left(f\left(x\right)\right)\in S^{(t)}\right].
\]

We can interpret the merge in Stage 2 in two ways depending on the
case. In Case 2, the algorithm moves the prediction of $S'$ from
$\text{pred}\left(S'\right)$ to $\ensuremath{\pi}\left(z^{(t)}\right)$.
In Case 3, it moves the prediction of $S^{(t)}$ from $\ensuremath{\pi}\left(z^{(t)}\right)$
to $\text{pred}\left(S'\right)$. By summing the squared error differences
over all time steps $t=0$ to $T$, we get that
\begin{align*}
 & \E\left[\left\Vert h_{T}(x)-y\right\Vert _{2}^{2}\right]-\E\left[\left\Vert h_{0}(x)-y\right\Vert _{2}^{2}\right]\\
 & \leq\sum_{t=0}^{T-1}\E\left[\left\Vert \ensuremath{\pi}\left(z^{(t)}\right)-y\right\Vert _{2}^{2}-\left\Vert h_{t}(x)-y\right\Vert _{2}^{2}\left|R\left(f\left(x\right)\right)\in S^{(t)}\right.\right]\P\left[R\left(f\left(x\right)\right)\in S^{(t)}\right]\\
 & \quad+\frac{4}{\lambda}\sum_{t=0}^{T-1}\P\left[R\left(f\left(x\right)\right)\text{is in the bin moved in Stage 2 of round }t\right].
\end{align*}

Let $\tau(v)$ denote the number of times the level set $v$ is in
the bin whose prediction gets moved in Stage 2. Then, $\sum_{t=0}^{T-1}\P\left[R\left(f\left(x\right)\right)\text{is in the bin moved in Stage 2 of round }t\right]=\sum_{v\in B}\P\left[R\left(f\left(x\right)\right)=v\right]\cdot\tau(v)$.

We now establish an upper bound on $\tau(v)$ for $v\in B$. Suppose
that $v$ is in the bin that gets moved in Stage 2 of some time step
$t$, during the merge bins $S_{a}$ and $S_{b}$. Without loss of
generality, assume that $S_{a}$ is the bin being moved. This implies
that $v\in S_{a}$ and $\sum_{S\in M:S\subseteq S_{a}}\hat{P}_{S}\leq\sum_{S\in M:S\subseteq S_{b}}\hat{P}_{S}$.
By the accuracy of the probability estimates, we have that $\P\left[R\left(f\left(x\right)\right)\in S_{a}\right]\leq\P\left[R\left(f\left(x\right)\right)\in S_{b}\right]+\nicefrac{\beta}{18}$.
Since $S_{a}$and $S_{b}$ are disjoint, $\P\left[R\left(f\left(x\right)\right)\in S_{a}\cup S_{b}\right]\geq\P\left[R\left(f\left(x\right)\right)\in S_{a}\right]-\nicefrac{\beta}{18}$.
Since each merge involving moving the bin with $v$ (almost) doubles
the size of the bin containing it, we have that 
\begin{align*}
2^{\tau(v)}\P\left[R\left(f\left(x\right)\right)=v\right]-\frac{\beta}{36}\sum_{i=1}^{\tau(v)}2^{i} & \leq1.
\end{align*}
Hence, 
\begin{align*}
\tau(v) & \leq\log_{2}\left(\frac{1-\nicefrac{\beta}{18}}{\P\left[R\left(f\left(x\right)\right)=v\right]-\nicefrac{\beta}{18}}\right).
\end{align*}
Since $\varepsilon<1$, we have $\beta=\varepsilon^{p/(p-1)}\cdot2^{-1/(p-1)}<1$.
Additionally, $\P\left[R\left(f\left(x\right)\right)=v\right]\geq\nicefrac{\beta}{6}-\nicefrac{\beta}{12}=\nicefrac{\beta}{12}$because
$v\in B$. Therefore, $\tau(v)\leq\log_{2}(36/\beta)$. Since $\lambda=\left\lceil \nicefrac{1}{\beta}\right\rceil ,$we
conclude that 
\begin{align*}
 & \E\left[\left\Vert h_{T}\left(x\right)-y\right\Vert _{2}^{2}\right]-\E\left[\left\Vert h_{0}\left(x\right)-y\right\Vert _{2}^{2}\right]\\
 & \leq\sum_{t=0}^{T-1}\E\left[\left\Vert \ensuremath{\pi}\left(z^{(t)}\right)-y\right\Vert _{2}^{2}-\left\Vert h_{t}\left(x\right)-y\right\Vert _{2}^{2}\left|R\left(f\left(x\right)\right)\in S^{(t)}\right.\right]\P\left[R\left(f\left(x\right)\right)\in S^{(t)}\right]\\
 & \quad+\frac{4}{\left\lceil \nicefrac{1}{\beta}\right\rceil }\log_{2}\left(\frac{36}{\beta}\right).
\end{align*}
\end{proof}
\begin{lem}[Lemma \ref{lem:sq-error-stage1} restated]
 Assuming that $A_{1},A_{2}$ and $A_{3}$ hold, at time step $t$
of the algorithm
\[
\E\left[\left\Vert \ensuremath{\pi}\left(z^{(t)}\right)-y\right\Vert _{2}^{2}-\left\Vert h_{t}\left(x\right)-y\right\Vert _{2}^{2}\left|R\left(f\left(x\right)\right)\in S^{(t)}\right.\right]\P\left[R\left(f\left(x\right)\right)\in S^{(t)}\right]\leq-\nicefrac{\beta^{2}}{9}.
\]
\end{lem}

\begin{proof}
At each time step $t\leq T-1$, before the algorithm terminates we
observe the following. Since $\pi\left(z^{(t)}\right)=\arg\min_{v\in\Delta_{k}}\left\Vert v-z^{(t)}\right\Vert _{2}$
and $y\in\Delta_{k}$, we have that $\left\Vert \pi\left(z^{(t)}\right)-y\right\Vert _{2}\leq\left\Vert z^{(t)}-y\right\Vert _{2}$.
Therefore, it suffices to find an upper bound for the following quantity:
\[
\E\left[\left\Vert z^{(t)}-y\right\Vert _{2}^{2}-\left\Vert h_{t}\left(x\right)-y\right\Vert _{2}^{2}\left|R\left(f\left(x\right)\right)\in S^{(t)}\right.\right]\P\left[R\left(f\left(x\right)\right)\in S^{(t)}\right].
\]
 For simplicity, let $u^{(t)}=\text{pred}\left(S^{(t)}\right)$ denote
the previous prediction for group $S^{(t)}$. Then we have that 
\begin{align*}
 & \E\left[\left\Vert z^{(t)}-y\right\Vert _{2}^{2}-\left\Vert u^{(t)}-y\right\Vert _{2}^{2}\left|R\left(f\left(x\right)\right)\in S^{(t)}\right.\right]\P\left[R\left(f\left(x\right)\right)\in S^{(t)}\right]\\
 & =\E\left[\left(z_{j^{(t)}}^{(t)}-y_{j^{(t)}}\right)^{2}-\left(u_{j^{(t)}}^{(t)}-y_{j^{(t)}}\right)^{2}\left|R\left(f\left(x\right)\right)\in S^{(t)}\right.\right]\P\left[R\left(f\left(x\right)\right)\in S^{(t)}\right]\\
 & =\left(\left(z_{j^{(t)}}^{(t)}\right)^{2}-\left(u_{j^{(t)}}^{(t)}\right)^{2}\right)\P\left[R\left(f\left(x\right)\right)\in S^{(t)}\right]+\left(2u_{j^{(t)}}^{(t)}-2z_{j^{(t)}}^{(t)}\right)\E\left[y_{j^{(t)}}\one\left[R\left(f\left(x\right)\right)\in S^{(t)}\right]\right]\\
 & =\left(z_{j^{(t)}}^{(t)}-u_{j^{(t)}}^{(t)}\right)\left(\left(z_{j^{(t)}}^{(t)}+u_{j^{(t)}}^{(t)}\right)\P\left[R\left(f\left(x\right)\right)\in S^{(t)}\right]-2\E\left[y_{j^{(t)}}\one\left[R\left(f\left(x\right)\right)\in S^{(t)}\right]\right]\right).
\end{align*}

The value of $z_{j^{(t)}}^{(t)}$, as assigned by the algorithm, falls
into one of two cases. Simultaneously, we have bounds on the value
of $u_{j^{(t)}}^{(t)}$, since the algorithm has selected a bin $S^{(t)}$
with large error. These bounds play a crucial role in analyzing
\[
\left(z_{j^{(t)}}^{(t)}-u_{j^{(t)}}^{(t)}\right)
\]
 and
\[
\left(\left(z_{j^{(t)}}^{(t)}+u_{j^{(t)}}^{(t)}\right)\P\left[R\left(f\left(x\right)\right)\in S^{(t)}\right]-2\E\left[y_{j^{(t)}}\one\left[R\left(f\left(x\right)\right)\in S^{(t)}\right]\right]\right).
\]

Case 1: $z_{j^{(t)}}^{(t)}=1$. Then, $\sum_{S\in M:S\subseteq S^{(t)}}\hat{E}_{S,j^{(t)}}\ge\sum_{S\in M:S\subseteq S^{(t)}}\hat{P}_{S}$
and $\left(\sum_{S\in M:S\subseteq S^{(t)}}\hat{P}_{S}\right)u_{j^{(t)}}^{(t)}-\sum_{S\in M:S\subseteq S^{(t)}}\hat{E}_{S,j^{(t)}}<-\nicefrac{\beta}{2}$.
Therefore,
\begin{align*}
 & \E\left[\left\Vert z^{(t)}-y\right\Vert _{2}^{2}-\left\Vert u^{(t)}-y\right\Vert _{2}^{2}\left|R\left(f\left(x\right)\right)\in S^{(t)}\right.\right]\P\left[R\left(f\left(x\right)\right)\in S^{(t)}\right]\\
 & =\left(1-u_{j^{(t)}}^{(t)}\right)\left(\left(1+u_{j^{(t)}}^{(t)}\right)\P\left[R\left(f\left(x\right)\right)\in S^{(t)}\right]-2\E\left[y_{j^{(t)}}\one\left[R\left(f\left(x\right)\right)\in S^{(t)}\right]\right]\right).
\end{align*}

We analyze the two factors separately. Since the error associated
with bin $S^{(t)}$ and coordinate $j^{(t)}$ is large, we have that
\begin{align*}
 & \left(\sum_{S\in M:S\subseteq S^{(t)}}\hat{P}_{S}\right)u_{j^{(t)}}^{(t)}\\
 & <\sum_{S\in M:S\subseteq S^{(t)}}\hat{E}_{S,j^{(t)}}-\frac{\beta}{2}\\
 & <\E\left[y_{j^{(t)}}\one\left[R\left(f\left(x\right)\right)\in S^{(t)}\right]\right]+\frac{\beta}{36(\left\lfloor \log_{2}|B|\right\rfloor +1)}\left|\left\{ S\in M:S\subseteq S^{(t)}\right\} \right|-\frac{\beta}{2}\\
 & \leq\P\left[R\left(f\left(x\right)\right)\in S^{(t)}\right]-\frac{17\beta}{36}.
\end{align*}
 Furthermore, we have a lower on the estimated probability of $S^{(t)}$
$\sum_{S\in M:S\subseteq S^{(t)}}\hat{P}_{S}\geq\P\left[R\left(f\left(x\right)\right)\in S^{(t)}\right]-\frac{\beta}{36(\left\lfloor \log_{2}|B|\right\rfloor +1)}\left|\left\{ S\in M:S\subseteq S^{(t)}\right\} \right|\geq\frac{\beta}{6}-\frac{\beta}{12}-\frac{\beta}{36}>0$
because $S^{(t)}\in G$, which implies that it contains bins from
set $B$.

Combining the two inequalities above, we obtain that 
\begin{align*}
1-u_{j^{(t)}}^{(t)} & >1-\frac{\P\left[R\left(f\left(x\right)\right)\in S^{(t)}\right]-\nicefrac{17\beta}{36}}{\P\left[R\left(f\left(x\right)\right)\in S^{(t)}\right]-\nicefrac{\beta}{36}}\\
 & =\frac{\nicefrac{\beta}{2}-\nicefrac{\beta}{18}}{\P\left[R\left(f\left(x\right)\right)\in S^{(t)}\right]-\nicefrac{\beta}{36}}>\frac{4\beta}{9}.
\end{align*}

We now bound the second factor.
\begin{align*}
 & \left(1+u_{j^{(t)}}^{(t)}\right)\P\left[R\left(f\left(x\right)\right)\in S^{(t)}\right]-2\E\left[y_{j^{(t)}}\one\left[R\left(f\left(x\right)\right)\in S^{(t)}\right]\right]\\
 & \leq\left(1+u_{j^{(t)}}^{(t)}\right)\left(\sum_{S\in M:S\subseteq S^{(t)}}\hat{P}_{S}+\frac{\beta}{36(\left\lfloor \log_{2}|B|\right\rfloor +1)}\left|\left\{ S\in M:S\subseteq S^{(t)}\right\} \right|\right)\\
 & \quad-2\left(\sum_{S\in M:S\subseteq S^{(t)}}\hat{E}_{S,j^{(t)}}-\frac{\beta}{36(\left\lfloor \log_{2}|B|\right\rfloor +1)}\left|\left\{ S\in M:S\subseteq S^{(t)}\right\} \right|\right)\\
 & \leq u_{j^{(t)}}^{(t)}\left(\sum_{S\in M:S\subseteq S^{(t)}}\hat{P}_{S}\right)-\sum_{S\in M:S\subseteq S^{(t)}}\hat{E}_{S,j^{(t)}}+\sum_{S\in M:S\subseteq S^{(t)}}\hat{P}_{S}-\sum_{S\in M:S\subseteq S^{(t)}}\hat{E}_{S,j^{(t)}}+\frac{\beta}{9}\\
 & <-\frac{7\beta}{18}.
\end{align*}

Multiplying the two factors, we see that
\begin{align*}
\E\left[\left\Vert z^{(t)}-y\right\Vert _{2}^{2}-\left\Vert u^{(t)}-y\right\Vert _{2}^{2}\left|R\left(f\left(x\right)\right)\in S^{(t)}\right.\right]\P\left[R\left(f\left(x\right)\right)\in S^{(t)}\right] & <-\frac{14\beta^{2}}{81}.
\end{align*}

At a high level, we have shown that the expected difference in squared
error is strictly negative in this case.

Case 2: $z_{j^{(t)}}^{(t)}=\left(\sum_{S\in M:S\subseteq S^{(t)}}\hat{E}_{S,j^{(t)}}\right)/\left(\sum_{S\in M:S\subseteq S^{(t)}}\hat{P}_{S}\right)\le1.$
We consider two subcases based on the behavior of $u_{j^{(t)}}^{(t)}$.

Subcase 1: $\sum_{S\in M:S\subseteq S^{(t)}}\hat{E}_{S,j^{(t)}}-\left(\sum_{S\in M:S\subseteq S^{(t)}}\hat{P}_{S}\right)u_{j^{(t)}}^{(t)}>\nicefrac{\beta}{2}$.
Then, it follows that
\[
z_{j^{(t)}}^{(t)}-u_{j^{(t)}}^{(t)}=\frac{\sum_{S\in M:S\subseteq S^{(t)}}\hat{E}_{S,j^{(t)}}}{\sum_{S\in M:S\subseteq S^{(t)}}\hat{P}_{S}}-u_{j^{(t)}}^{(t)}>\frac{\beta}{2\left(\sum_{S\in M:S\subseteq S^{(t)}}\hat{P}_{S}\right)}
\]
 and

\begin{align*}
 & \left(z_{j^{(t)}}^{(t)}+u_{j^{(t)}}^{(t)}\right)\P\left[R\left(f\left(x\right)\right)\in S^{(t)}\right]-2\E\left[y_{j^{(t)}}\one\left[R\left(f\left(x\right)\right)\in S^{(t)}\right]\right]\\
 & =\left(\frac{\sum_{S\in M:S\subseteq S^{(t)}}\hat{E}_{S,j^{(t)}}}{\sum_{S\in M:S\subseteq S^{(t)}}\hat{P}_{S}}+u_{j^{(t)}}^{(t)}\right)\P\left[R\left(f\left(x\right)\right)\in S^{(t)}\right]-2\E\left[y_{j^{(t)}}\one\left[R\left(f\left(x\right)\right)\in S^{(t)}\right]\right]\\
 & <\left(2\frac{\sum_{S\in M:S\subseteq S^{(t)}}\hat{E}_{S,j^{(t)}}}{\sum_{S\in M:S\subseteq S^{(t)}}\hat{P}_{S}}-\frac{\beta}{2\sum_{S\in M:S\subseteq S^{(t)}}\hat{P}_{S}}\right)\P\left[R\left(f\left(x\right)\right)\in S^{(t)}\right]-2\E\left[y_{j^{(t)}}\one\left[R\left(f\left(x\right)\right)\in S^{(t)}\right]\right]\\
 & \leq\left(2\frac{\sum_{S\in M:S\subseteq S^{(t)}}\hat{E}_{S,j^{(t)}}}{\sum_{S\in M:S\subseteq S^{(t)}}\hat{P}_{S}}-\frac{\beta}{2\sum_{S\in M:S\subseteq S^{(t)}}\hat{P}_{S}}\right)\\
 & \quad\cdot\left(\sum_{S\in M:S\subseteq S^{(t)}}\hat{P}_{S}+\frac{\beta}{36(\left\lfloor \log_{2}|B|\right\rfloor +1)}\left|\left\{ S\in M:S\subseteq S^{(t)}\right\} \right|\right)\\
 & \quad-2\left(\sum_{S\in M:S\subseteq S^{(t)}}\hat{E}_{S,j^{(t)}}-\frac{\beta}{36(\left\lfloor \log_{2}|B|\right\rfloor +1)}\left|\left\{ S\in M:S\subseteq S^{(t)}\right\} \right|\right)\\
 & \leq-\frac{\beta}{2}-\frac{\beta^{2}}{2\cdot36(\left\lfloor \log_{2}|B|\right\rfloor +1)\sum_{S\in M:S\subseteq S^{(t)}}\hat{P}_{S}}\left|\left\{ S\in M:S\subseteq S^{(t)}\right\} \right|+\frac{\beta}{18}<-\frac{4\beta}{9}.
\end{align*}

Subcase 2: $\sum_{S\in M:S\subseteq S^{(t)}}\hat{E}_{S,j^{(t)}}-\left(\sum_{S\in M:S\subseteq S^{(t)}}\hat{P}_{S}\right)u_{j^{(t)}}^{(t)}<-\nicefrac{\beta}{2}$.
Then, it follows that
\[
z_{j^{(t)}}^{(t)}-u_{j^{(t)}}^{(t)}=\frac{\sum_{S\in M:S\subseteq S^{(t)}}\hat{E}_{S,j^{(t)}}}{\sum_{S\in M:S\subseteq S^{(t)}}\hat{P}_{S}}-u_{j^{(t)}}^{(t)}<-\frac{\beta}{2\left(\sum_{S\in M:S\subseteq S^{(t)}}\hat{P}_{S}\right)}
\]
 and
\begin{align*}
 & \left(z_{j^{(t)}}^{(t)}+u_{j^{(t)}}^{(t)}\right)\P\left[R\left(f\left(x\right)\right)\in S^{(t)}\right]-2\E\left[y_{j^{(t)}}\one\left[R\left(f\left(x\right)\right)\in S^{(t)}\right]\right]\\
 & =\left(\frac{\sum_{S\in M:S\subseteq S^{(t)}}\hat{E}_{S,j^{(t)}}}{\sum_{S\in M:S\subseteq S^{(t)}}\hat{P}_{S}}+u_{j^{(t)}}^{(t)}\right)\P\left[R\left(f\left(x\right)\right)\in S^{(t)}\right]-2\E\left[y_{j^{(t)}}\one\left[R\left(f\left(x\right)\right)\in S^{(t)}\right]\right]\\
 & >\left(2\frac{\sum_{S\in M:S\subseteq S^{(t)}}\hat{E}_{S,j^{(t)}}}{\sum_{S\in M:S\subseteq S^{(t)}}\hat{P}_{S}}+\frac{\beta}{2\sum_{S\in M:S\subseteq S^{(t)}}\hat{P}_{S}}\right)\\
 & \quad\cdot\left(\sum_{S\in M:S\subseteq S^{(t)}}\hat{P}_{S}-\frac{\beta}{36(\left\lfloor \log_{2}|B|\right\rfloor +1)}\left|\left\{ S\in M:S\subseteq S^{(t)}\right\} \right|\right)\\
 & \quad-2\left(\sum_{S\in M:S\subseteq S^{(t)}}\hat{E}_{S,j^{(t)}}+\frac{\beta}{36(\left\lfloor \log_{2}|B|\right\rfloor +1)}\left|\left\{ S\in M:S\subseteq S^{(t)}\right\} \right|\right)\\
 & \geq\frac{\beta}{2}-\frac{\beta^{2}}{2\cdot36(\left\lfloor \log_{2}|B|\right\rfloor +1)\sum_{S\in M:S\subseteq S^{(t)}}\hat{P}_{S}}\left|\left\{ S\in M:S\subseteq S^{(t)}\right\} \right|-\frac{\beta}{18}>\frac{4\beta}{9}.
\end{align*}

Therefore, in both subcases the expected difference in squared error
is also strictly negative. Specifically, we have
\begin{align*}
 & \E\left[\left\Vert z^{(t)}-y\right\Vert _{2}^{2}-\left\Vert u^{(t)}-y\right\Vert _{2}^{2}\left|R\left(f\left(x\right)\right)\in S^{(t)}\right.\right]\P\left[R\left(f\left(x\right)\right)\in S^{(t)}\right]\\
 & <-\left(\frac{4\beta}{9}\right)\frac{\beta}{2\left(\sum_{S\in M:S\subseteq S^{(t)}}\hat{P}_{S}\right)}\\
 & <-\frac{\beta^{2}}{9}.
\end{align*}
 because $\sum_{S\in M:S\subseteq S^{(t)}}\hat{P}_{S}\leq\P\left[R\left(f\left(x\right)\right)\in S^{(t)}\right]+\frac{\beta}{36}\leq2$.

We notice that in both cases 
\[
\E\left[\left\Vert z^{(t)}-y\right\Vert _{2}^{2}-\left\Vert u^{(t)}-y\right\Vert _{2}^{2}\left|R\left(f\left(x\right)\right)\in S^{(t)}\right.\right]\P\left[R\left(f\left(x\right)\right)\in S^{(t)}\right]<-\frac{\beta^{2}}{9}.
\]
\end{proof}
\begin{lem}[Lemma \ref{lem:sq-error-t=00003D0} restated]
 The squared error at time step $0$ is
\[
\E\left[\left\Vert h_{0}(x)-y\right\Vert _{2}^{2}\right]\leq\E\left[\left\Vert f\left(x\right)-y\right\Vert _{2}^{2}\right]+O\left(\beta\right).
\]
\end{lem}

\begin{proof}
By the definition of $\rho$ , $h_{0}(x)=\rho\left(R\left(f\left(x\right)\right)\right)$
and $f(x)$ correspond to the same level set when they get rounded
by $R$. Therefore, they are at most $\nicefrac{1}{\lambda}$ apart
in every coordinate. Additionally, the coordinates of $f(x)$ and
$h_{0}(x)$ add up to $1$. Since $y$ is the one-hot encoding of
a label, we obtain that
\begin{align*}
 & \left\Vert h_{0}(x)-y\right\Vert _{2}^{2}\\
 & =\left\Vert h_{0}(x)-y\right\Vert _{2}^{2}-\left\Vert f(x)-y\right\Vert _{2}^{2}+\left\Vert f(x)-y\right\Vert _{2}^{2}\\
 & \leq\|h_{0}(x)\|_{2}^{2}-\|f(x)\|_{2}^{2}+2\max_{j\in[k]}|h_{0}(x)_{j}-f(x)_{j}|+\left\Vert f(x)-y\right\Vert _{2}^{2}\\
 & \leq\left(\max_{j\in[k]}\left|h_{0}(x)_{j}-f(x)_{j}\right|\right)\sum_{j\in[k]}\left(\left|h_{0}(x)_{j}\right|+\left|f(x)_{j}\right|\right)+\left\Vert f(x)-y\right\Vert _{2}^{2}\\
 & \leq\frac{1}{\lambda}\cdot4+\left\Vert f(x)-y\right\Vert _{2}^{2}=\frac{4}{\left\lceil \nicefrac{1}{\beta}\right\rceil }+\left\Vert f(x)-y\right\Vert _{2}^{2}.
\end{align*}
\end{proof}

\section{Proofs from Subsection \ref{subsec:Termination}}
\begin{lem}[Lemma \ref{lem:algorithm-termination} restated]
 Assuming that $A_{1}$, $A_{2}$, and $A_{3}$ hold, the algorithm
terminates after at most $O\left(\nicefrac{1}{\beta^{2}}\right)$
time steps with probability at least $1-\delta$.
\end{lem}

\begin{proof}
Assuming that events $A_{1},A_{2}$ and $A_{3}$ hold, we apply Lemmata
\ref{lem:sq-error-stage2} and \ref{lem:sq-error-stage1} to obtain
the following bound
\[
\E\left[\left\Vert h(x)-y\right\Vert _{2}^{2}\right]-\E\left[\left\Vert \rho\left(R\left(f\left(x\right)\right)\right)-y\right\Vert _{2}^{2}\right]\leq-\frac{\beta^{2}}{9}T+\frac{4}{\left\lceil \nicefrac{1}{\beta}\right\rceil }\log_{2}\left(\frac{36}{\beta}\right).
\]

Moreover , since the squared loss is always bounded between 0 and
1 we have 
\[
-1\leq-\frac{\beta^{2}}{9}T+\frac{4}{\left\lceil \nicefrac{1}{\beta}\right\rceil }\log_{2}\left(\frac{36}{\beta}\right)
\]
which implies that the algorithm must terminate after 
\[
T\leq\frac{9+\frac{36}{\left\lceil \nicefrac{1}{\beta}\right\rceil }\log_{2}\left(\frac{36}{\beta}\right)}{\beta^{2}}
\]
time steps.
\end{proof}
\begin{lem}[Lemma \ref{lem:termination-error} restated]
 Assuming that $A_{1}$, $A_{2}$, and $A_{3}$ hold, the $\ell_{p}$
calibration error of $h$ $\left(\textrm{Err}_{p}(h)\right)^{p}$
is bounded by $O(\beta^{p-1})$ with probability at least $1-\delta$.
\end{lem}

\begin{proof}
Let $T$ be the time step when the algorithm terminates. First, we
analyze the error under the assumption that $A_{1},A_{2}$ and $A_{3}$
hold. We show that for all $v\in\VV$ and all $j\in[k]$, $\textrm{Err}(h,v,j)\le\beta$.

A point $x$ gets a prediction $h(x)$ that gets rounded to level
set $v$ in one of two ways:

1. if $v$ is not a high-probability bin, then the initial prediction
$f(x)$ gets rounded to $v$, or

2. if there exists a group of bins $S\in G$ such that $R\left(\text{pred}(S)\right)=v$,
then the initial prediction $f(x)$ is in a high-probability bin that,
through the calibration algorithm gets mapped to group $S$.

Note that both cases can be true simultaneously for a fixed $v$.
In the second case, due to the termination criterion of the algorithm,
$\forall j\in[k]$, 
\[
\hat{\textrm{Err}}(S,j)=\left|(\sum_{S'\in M:S'\subseteq S}\hat{P}_{S'})\text{pred}\left(S\right)_{j}-\sum_{S'\in M:S'\subseteq S}\hat{E}_{S',j}\right|\leq\frac{\beta}{2}.
\]
For the true error of $v\in\VV$ and $j\in[k]$, we have that
\begin{align*}
 & \textrm{Err}(h,v,j)\\
 & =\left|\E_{(x,y)\sim D}\left[\left(h(x)_{j}-y_{j}\right)\one\left[R\left(h\left(x\right)\right)=v\right]\right]\right|\\
 & \leq\left|\E_{(x,y)\sim D}\left[\left(h(x)_{j}-y_{j}\right)\one\left[R\left(h\left(x\right)\right)=v\textrm{ and }R(f(x))\in B\right]\right]\right|\\
 & \quad+\left|\E_{(x,y)\sim D}\left[\left(h(x)_{j}-y_{j}\right)\one\left[R\left(h\left(x\right)\right)=v\textrm{ and }R(f(x))\notin B\right]\right]\right|\\
 & \leq\left|\P\left[R(f(x))\in S\right]\cdot\text{pred}(S)_{j}-\E_{(x,y)\sim D}\left[y_{j}\one\left[R(f(x))\in S\right]\right]\right|\cdot\\
 & \;\one\left[\exists S\in G:R\left(\text{pred}(S)\right)=v\right]+\P\left[R(f(x))=v\right]\one\left[v\notin B\right]\\
 & \leq(|(\sum_{S'\in M:S'\subseteq S}\hat{P}_{S'})\text{pred}(S)_{j}-\sum_{S'\in M:S'\subseteq S}\hat{E}_{S',j}|\\
 & \;+\frac{2\beta}{36(\left\lfloor \log_{2}|B|\right\rfloor +1)}\left|\left\{ S'\in M:S'\subseteq S\right\} \right|)\one\left[\exists S\in G:R\left(\text{pred}(S)\right)=v\right]+\frac{\beta}{6}\one[v\notin B]\\
 & \leq\left(\frac{\beta}{2}+\frac{\beta}{18}\right)\one\left[\exists S\in G:R\left(\text{pred}(S)\right)=v\right]+\frac{\beta}{6}\one[v\notin B]\\
 & \leq\beta
\end{align*}
Therefore,
\begin{align*}
 & \sum_{v\in V_{\lambda}^{k}}\sum_{j=1}^{k}(\textrm{Err}(h,v,j))^{p}\\
 & \le(\sum_{v\in V_{\lambda}^{k}}\sum_{j=1}^{k}\textrm{Err}(h,v,j))\max_{v\in V_{\lambda}^{k},j\in[k]}(\textrm{Err}(h,v,j))^{p-1}\\
 & \le(\sum_{v\in V_{\lambda}^{k}}\sum_{j=1}^{k}\left(\E_{(x,y)\sim D}\left[h(x)_{j}\mid R(h(x))=v\right]\right.\\
 & \left.+\E_{(x,y)\sim D}\left[y_{j}\mid R(h(x))=v\right]\right)\P\left[R(h(x))=v\right])\beta^{p-1}\\
 & \le2\beta^{p-1}
\end{align*}
This holds because for all $v\in\VV$, $\sum_{j=1}^{k}\E_{(x,y)\sim D}\left[h(x)_{j}\mid R(h(x))=v\right]=1$.
As a result we get that $\P\left[\textrm{Err}_{p}(h)>\left(2\beta^{p-1}\right)^{1/p}\left|A_{1},A_{2},A_{3}\right.\right]=0$. 
\end{proof}
\begin{lem}[Lemma \ref{lem:runtime} restated]
 Assuming that $A_{1}$, $A_{2}$, and $A_{3}$ hold, the algorithm
terminates in time polynomial in $\frac{1}{\beta}$ and $k$.
\end{lem}

\begin{proof}
Assuming that $A_{1}$, $A_{2}$ and $A_{3}$ hold, Algorithm \ref{alg:multiclass_cal_full}
has time complexity $O\left(\text{poly}\left(\frac{1}{\beta},k\right)\right)$,
where $\text{poly}$ denotes a polynomial function. We analyze the
time complexity of each phase of the algorithm.

Phase 1: Identifying high-probability bins. This phase requires $O(n)$
time, where $n$ is the number of samples used to estimate $\hat{\mu}_{v}$.
According to the analysis in Subsection \ref{subsec:Correctness-of-Estimates},
$n$ is polynomial in $\frac{1}{\beta}$ and logarithmic in $k$.
Notably, this step avoids iterating over all bins in $V_{\lambda}^{k}$
by examining only bins containing input samples. This can be efficiently
implemented using a dictionary/hash table where keys represent bins
and values are lists of samples in each bin. The dictionary size equals
the number of non-empty bins. From this point forward the algorithm
operates exclusively on the high probability bins in $B$, whose cardinality
is linear in $\frac{1}{\beta}$.

Phase 2: Initializing data structures $M$ and $G$. This requires
time polynomial in |$B$|, $k$, and the number of samples used for
estimating $\hat{P}$ and $\hat{E}$. By the analysis in Subsection
\ref{subsec:Correctness-of-Estimates}, the number of these samples
is polynomial in $\frac{1}{\beta}$ and logarithmic in $k$.

Phase 3: Calibration. The algorithm calibrates predictions for bins
in $B$ by executing at most $O\left(\frac{1}{\beta^{2}}\right)$
iterations. Each iteration performs a polynomial number of operations
in $k$ and $\frac{1}{\beta}$. More specifically, searching in $G$
for the large-error group requires at most $|B|$ time. The total
number of merges in $G$ and $M$ throughout the entire algorithm
is bounded by $|B|$, since we begin with $|B|$ groups and only merge.
The estimation and error computation steps run in time polynomial
in the sample size and $k$. 

Combining the analyses of the three phases, we conclude that the algorithm's
time complexity is polynomial in $\frac{1}{\beta}$ and $k$.
\end{proof}

\end{document}